\documentclass[twoside]{article}

\usepackage[accepted]{aistats2022}

\usepackage[utf8]{inputenc} %
\usepackage[T1]{fontenc}    %
\usepackage{hyperref}       %
\usepackage{url}            %
\usepackage{booktabs} 
\usepackage{amsmath}
\RequirePackage{amssymb}
\RequirePackage{amsthm}
\usepackage{mathabx}
\usepackage{amssymb}
\usepackage{mathtools}
\usepackage[sort]{cite}
\usepackage{subfigure}
\usepackage{amsfonts}

\DeclarePairedDelimiterX{\lin}[2]{\langle}{\rangle}{#1, #2}

\DeclarePairedDelimiterX{\abs}[1]{\lvert}{\rvert}{#1}
\DeclarePairedDelimiterX{\norm}[1]{\lVert}{\rVert}{#1}
\DeclarePairedDelimiterX{\cbr}[1]{\{}{\}}{#1} %
\DeclarePairedDelimiterX{\rbr}[1]{(}{)}{#1} %
\DeclarePairedDelimiterX{\sbr}[1]{[}{]}{#1} %

\DeclareMathOperator{\sgn}{sign}
\makeatletter
\def\sign{\@ifnextchar*{\@sgnargscaled}{\@ifnextchar[{\sgnargscaleas}{\@ifnextchar{\bgroup}{\@sgnarg}{\sgn} }}}
\def\@sgnarg#1{\sgn\rbr{#1}}
\def\@sgnargscaled#1{\sgn\rbr*{#1}}
\def\@sgnargscaleas[#1]#2{\sgn\rbr[#1]{#2}}
\makeatother

\providecommand{\bI}{\mathbf{I}}

\providecommand{\xx}{\mathbf{x}}
\providecommand{\yy}{\mathbf{y}}

\newtheorem{lemma}{Lemma}

\newcommand{\E}{\mathbb{E}}

\newcommand{\Ea}[1]{\E\left[#1\right]}
\newcommand{\Eb}[2]{\E_{#1}\left[#2\right]}

\newcommand{\Prob}[1]{\operatorname{Pr}\left[#1\right]}

\def\chi{{y}}

\usepackage{bm}

\usepackage{nicefrac}       %
\usepackage{microtype}      %
\usepackage{natbib}
\usepackage{tikz, graphicx}

\newtheorem{cor}{Corollary}
\newtheorem{thm}{Theorem}

\providecommand\df[2]{\frac{\partial #1}{\partial #2}}

\newcommand{\bw}{\bm{w}}

\newcommand{\bS}{\bm{S}}

\usepackage{bm}

\begin{document}
\twocolumn[

\aistatstitle{Understanding Layer-wise Contributions in Deep Neural Networks through Spectral Analysis}

\aistatsauthor{ Yatin Dandi \And Arthur Jacot }

\aistatsaddress{IIT Kanpur, India, \\
Ecole Polytechnique Fédérale de Lausanne,\\ Switzerland,\\ yatind@iitk.ac.in
\And Chair of Statistical Field Theory,\\
Ecole Polytechnique Fédérale de Lausanne,\\ Switzerland,\\ arthur.jacot@epfl.ch} ]

\begin{abstract}
Spectral analysis is a powerful tool, decomposing any function into simpler parts. In machine learning, Mercer's theorem generalizes this idea, providing for any kernel and input distribution a natural basis of functions of increasing frequency.
More recently, several works have extended this analysis to deep neural networks through the framework of Neural Tangent Kernel. 
In this work, we  analyze the layer-wise spectral bias of Deep Neural Networks and relate it to the contributions of different layers in the reduction of generalization error for a given target function.  We utilize the properties of Hermite polynomials and Spherical Harmonics to prove that initial layers exhibit a larger bias towards high frequency functions defined on the unit sphere. We further provide empirical results validating our theory in high dimensional datasets for Deep Neural Networks.
\end{abstract}

\section{Introduction}

Several recent theoretical and empirical advances indicate that understanding generalization in deep learning requires incorporating the properties of the data distributions as well as the optimization algorithms \citep{DBLP:conf/iclr/ZhangBHRV17}.
One view to interpret the generalization capabilities of deep neural networks is that the training dynamics of DNNs result in successive layers capturing functions invariant to high frequency pertubations or actions in the input space while incorporating low frequency functions such as classes.

Such properties in the trained models are desirable since most realistic data distributions correspond to data manifolds of objects such as the set of real images along with groups composed of actions such as rotations and shifting of objects acting on the data. A network that incorporates the invariance of properties such as classes and presence of objects to such group actions is expected to have better generalization capabilties.
A dual view of understanding the changes in functions under such actions is through fourier analysis on the corresponding groups.

As a step towards explaining the generalization behaviour of Deep Neural Networks, \citet{pmlr-v97-rahaman19a} highlighted the intriguing phenomenon of ``Spectral Bias'' in Relu networks where they empirically demonstrated the bias of Relu networks towards learning low frequency functions.

While a theoretical understanding of the ``Spectral Bias'' in finite-width neural networks remains elusive, several recent works have attempted to shed light into the phenomenon of spectral bias through the properties of the Neural Tangent Kernel (NTK)'s spectrum.
Neural Tangent Kernel (NTK) \citep{jacot2018neural} describes the time evolution of the output function's value at an input  through the similarity between gradients of the outputs at the training points and at the given input. In the infinite-width limit, under suitable scaling, the NTK converges to a fixed kernel.  
The evolution of the output function of the Deep Neural Network in such regimes corresponds to Kernel regression using the NTK. This enables understanding the optimization and generalization properties of the training dynamics by studying the properties of the corresponding NTK. 
Recent theoretical and empirical results have demonstrated that the NTK has high eigenvalues for functions corresponding  low frequency variations in the input space. The high eigenvalues of such low frequency functions leads to faster convergence along directions in function space having smoother variations in the input space. 

The relationship between the spectral properties of Kernels and generalization is well known, and has been used to derive explicit generalization bounds \citep{DBLP:conf/colt/BartlettM01} for Kernel methods. However, we argue that for the case of Deep Neural Networks, it can provide insights into not only the generalization capabilities but also the role played by different layers as well as the spectral properties of the learnt features.

Our analysis reveals that the eigenvalues corresponding to high frequency functions of the contribution to the NTK from initial layers are larger than those corresponding to latter layers. This is primarily due to application of the differentiation operation to the activation function while back-propagating the gradient through latter layers, leading the amplification of high frequency components.
This explains the predisposition of the initial layers towards contributing to the learning of high frequency functions. For example, in CNNs, the initial layers often detect high frequency artifacts such as edges whereas the latter layers detect smoother properties such as the presence of class.

By decomposing the NTK of the full network into  layer-wise contributions, we characterize the contributions of different layers to the decrease in the training cost.
Furthermore, we prove that the ratio of contributions of different layers to the decrease in generalization error along a direction in the function space when training on a finite number of data points, can be approximately described by the ratio of the function's squared norm under the inner product defined by the two Kernels
This ratio differs across functions of different frequencies. By exploiting the shared basis of eigenvectors for the kernels corresponding to different layers, we relate this ratio for functions of different frequencies to the corresponding ratio of eigenvalues. We then derive a general approach to analyze these ratios without explicitly computing the eigenvalues.

\section{Setup and Notation}
We consider the setup of a deep neural network (DNN)
having layers numbered $0,1,\cdots,L-1$, being training under gradient descent on a finite number of training points $\xx_1,\xx_2,\cdots,\xx_n$.
We use $f(\xx)$ and $\theta$ to denote the DNN's output function's value at a point $\xx$ and the set of parameters respectively.

\section{Preliminaries: Neural Tangent Kernel}

Our analysis relies on the framework of Neural Tangent Kernel, proposed by \citet{jacot2018neural}, who showed that under appropriate scaling, the training dynamics of neural networks in the limit of infinite-width can be described by a fixed kernel $\kappa_{NTK}$, whose value at two given points  $\xx$ and $\xx'$ is simply the inner product of the output's gradients w.r.t the parameters $\theta$ at the given two points, i.e:
\begin{align*}
    \kappa_{NTK}(\xx,\xx') = \langle\nabla_\theta f(\xx), \nabla_\theta f(\xx') \rangle, 
\end{align*}
The above inner product arises naturally when considering the evolution of the output function's value at a point $\xx$, upon training using a set of points $\xx_1,\xx_2\,\dots,\xx_n$. For instance, for the regression task, with the corresponding target values $y_1,y_2\,\dots,y_n$  the output evolves as:
\begin{align*}
    \frac{\partial f(\xx)}{\partial t} = -\frac{1}{n}\sum_{i=1}^n \langle\nabla_\theta f(\xx), \nabla_\theta f(\xx_i) \rangle (f(\xx_i)-y_i)
\end{align*}
Thus the spectral analysis of the NTK can reveal sensitivity of the output's gradients w.r.t the parameters upon training with different target functions. 

\section{Related Work}

A number of works have analyzed the spectrum of Neural Tangent Kernel under different conditions such as two layer relu networks without bias \citep{bietti2019inductive,cao2019understanding} under uniform input distribution on the sphere, in the presence of bias and non-uniform density \citep{pmlr-v119-basri20a} and on boolean input space \citep{yang2020finegrained}.
More recently, \citep{bietti2021deep} utilized the regularity of the kernels corresponding to Relu networks to prove that the NTK's spectrum for $n$-layer Relu networks has the same asymptotic decay as two layer Relu networks and the Laplace Kernel. The equivalence between the NTK for Relu networks and the Laplace Kernel was also derived independently by \citet{chen2021deep}.
Our work, instead focuses on the relative contributions of different layers to the NTK, to enable understanding the propagation of output gradients throughout the network.
Moreover, our results are applicable for arbitrary activation functions satisying minor smoothness assumptions.
The spectral bias during training for DNNs was initially studied in the Fourier domain by \citet{pmlr-v97-rahaman19a} and \citet{Xu_2020} for MNIST and toy datasets. \citet{valle-perez2018deep} highlighted a different kind of simplicity bias, by demonstrating that the mapping from the parameters to the output functions 
is biased towards simpler functions. However, unlike their analysis based solely on the output functions at initialization, the Neural Tangent Kernel (NTK) framework offers the advantage of incorporating the properties of the training algorithm, namely gradient descent while still being dependent only on the network at initialization.
While the entire network's spectrum explains the network's bias towards learning low frequency functions, and consequently the improved generalization performance for ``smooth'' target functions, the layer wise contributions explain how some layers are biased towards learning higher or lower frequencies in comparatively with the other layers.
A number of  works \citep{DBLP:conf/iclr/SchoenholzGGS17} have analyzed the Forward information propagation in Deep Neural Networks by recursively describing the covariance at layer $\ell$ in terms of layer $\ell-1$ to determine when the ratio $1$ of covariance is a stable fixed point. 
Our work instead sheds light into the propagation of gradients across layers and the sensitivity of the output to changes in different layers.
Our work also builds upon a long line of work on the analysis of dot-product Kernels \citep{Smoladot,pmlr-v130-scetbon21b}.

\section{Analysis}

\subsection{Ratio of Decrements}
Let $\theta_\ell \in \theta$ denote the the subset of parameters corresponding to the $\ell_{th}$ layer. 
Since the NTK's value at given points $\xx,\xx'$ corresponds to the inner product of the output values w.r.t the parameters $\theta$, it can be expressed as a sum of the inner products for $\theta_1,\theta_2,\cdots,\theta_{L-1}$ as follows:
\begin{align*}
    \langle\nabla_\theta f(\xx), \nabla_\theta f(\xx') \rangle = \sum_{\ell=0}^{L-1} \langle\nabla_{\theta_\ell} f(\xx), \nabla_{\theta_\ell} f(\xx') \rangle 
\end{align*}
The NTK for the entire network can thus be decomposed into the contributions from each layer to the NTK as follows:
\begin{equation}\label{eq:decom}
    \kappa_{NTK}(\xx,\xx') = \sum_{\ell=0}^{L-1}  \kappa^{(\ell)}(\xx,\xx')
\end{equation}

Following \citep{jacot2018neural}, we represent the cost function $C$ as a functional defined on a function space $\mathcal{F}$ on training inputs. A given cost functional $C$ and a given output function $f_t$ then defines an element $d|_{f(t)}$ in the dual space $\mathcal{F}*$ of functionals w.r.t the input distribution $p_{in}$ such that $ \partial_f^{in} C|_{f_t}  = \langle d |_{f_t},  \cdot \rangle_{p^{in}} $.
The evolution of the cost $C$ thus described by \citep{jacot2018neural}:
\begin{equation}\label{eq:funct}
\begin{split}
    	\partial_t C|_{f(t)} &= -\left<d|_{f(t)}, \nabla_{\kappa^{(L)}_{NTK}}C|_{f(t)} \right>_{p^{in}}\\&=-\left\|d|_{f(t)} \right\|_{\kappa^{(L)}_{NTK}}^{2},
\end{split}
\end{equation}
where the inner product w.r.t a kernel $\kappa$ is defined as $
\left<f, g \right>_\kappa := \mathbb{E}_{\xx, \xx' \sim p^{in}} \left[f(\xx)^T \kappa(\xx, \xx') g(\xx') \right].
$
To isolate the effect of each layer in the evolution of the training loss, we define $d^{(\ell)}$ as the contribution to the evolution of the training loss by the $\ell_{th}$ layer:
\begin{equation}
    d^{(\ell)}(d|_{f(t)}) = \left\|d|_{f(t)} \right\|_{\kappa^{(\ell)}}^{2}
\end{equation}
The ratio between the contributions for layers $\ell_1$ and $\ell_2$ to the decrements of the cost are thus given by:
\begin{equation}\label{eq:ratio}
\frac{d^{(\ell_1)}(d|_{f(t)})}{d^{(\ell_2)}(d|_{f(t)})}= \frac{\left\|d|_{f(t)} \right\|_{\kappa^{(\ell_1)}}^{2}}{\left\|d|_{f(t)} \right\|_{\kappa^{(\ell_2)}}^{2}}
\end{equation}

For simplicity, we consider $p_{in}$ to be the $U(\mathbb{S}^{d-1})$ i.e the cost is computed w.r.t a uniform distribution on the sphere. Moreover, as we demonstrate in Section \ref{sec:finite_ratio}, using standard concentration arguments, the ratio of decrements for finite training points and a function on the given training points can be related to the corresponding ratio of decrements for the integral operators and the associated eigenfunction.
\subsection{Relation to Mercer Decomposition}
Given an input space $X$ and a measure $\mu$ such that the space of square integrable functions along with the corresponding inner product constitute a Hilbert space, Kernels defined on the space can be interpreted as symmetric Hilbert Schmidt Operators. By Mercer's theorem, continuous positive symmetric operators acting on functions defined on an compact input space can be diagonalized by a countable orthonormal basis, known as the Mercer Decomposition, which plays a similar role to a spectral decomposition. 
Suppose $d|_{f(t)}$ (the functional defined in Equation \ref{eq:funct}) is continuous and bounded, so that it lies in the Hilbert space. 
Furthermore, suppose that it is an  
an eigenvector of both $\kappa^{(\ell_1)}$ and $\kappa^{(\ell_2)}$ with eigenvalues $\lambda_1,\lambda_2$ respectively. We obtain 
\begin{equation}\label{eq:eigen_ratio}
\frac{d^{(\ell_1)}(d|_{f(t)})}{d^{(\ell_2)}(d|_{f(t)})}= \frac{\left\|d|_{f(t)} \right\|_{\kappa^{(\ell_1)}}^{2}}{\left\|d|_{f(t)} \right\|_{\kappa^{(\ell_2)}}^{2}} = \frac{\lambda_1}{\lambda_2}
\end{equation}
Thus, the ratio of decrements for such a direction in the function space is simply given by the ratio of the corresponding eigenvalues.

\subsection{Finite training data}\label{sec:finite_ratio}

Let $g(\cdot)$ denote an arbitrary function, corresponding to a given target direction in the function space.
Consider the functions $\phi_1(\xx_i)=\int g(\xx)\kappa^{(\ell_1)}(\xx,\xx_i)g(\xx_i)d\xx$ and  $\phi_2(\xx_i)=\int g(\xx)\kappa^{(\ell_2)}(\xx,\xx_i)g(\xx_i)d\xx$, that describe the contributions to the decrease in the expected risk, of $\kappa^{(\ell_1)}$ and $\kappa^{(\ell_2)}$ respectively due to a single training point $x_i$. Thus, $\frac{1}{n} \sum_{i=1}^n \phi_1(\xx_i)$ and $\frac{1}{n} \sum_{i=1}^n \phi_2(\xx_i)$ describe the corresponding contributions while training on the finite collection of training points with inputs $\xx_1,\xx_2,\cdots,\xx_n$. Since the inputs lie on a compact space (the unit sphere), for continuous kernels $\kappa^{(\ell_1)}$ and $\kappa^{(\ell_2)}$ and function $g$, their magnitudes are bounded, say with constants $A,B$ and $C$ respectively.
Thus $\phi_1$ and $\phi_2$ are bounded as well, since we have:
\begin{align*}
    \phi_1(\xx_i)&= \int g(\xx)\rbr*{\kappa^{(\ell_1)}(\xx,\xx_i)g(\xx'_i)}d\xx
    \leq AC^2
\end{align*}
for some constant $C$ independent of $n$. Similarly, for $\phi_2$, we have:
\begin{align*}
    \phi_2(\xx_i)
    \leq BC^2
\end{align*}
We then utilize Hoeffding’s Inequality to obtain the following two bounds:
\begin{align*}
    &\Prob{\abs{\frac{1}{n} \sum_{i=1}^n \phi_1(\xx_i)-\iint g(\xx)\kappa^{(\ell_1)}(\xx,\xx')g(\xx') d\xx d\xx'}
    \geq t}\\&\leq 2e^{-2\frac{nt^2}{AC^2}}\\
    &\Prob{\abs{\frac{1}{n} \sum_{i=1}^n \phi_2(\xx_i)-\iint g(\xx)\kappa^{(\ell_1)}(\xx,\xx')g(\xx') d\xx d\xx'}
    \geq t}\\ &\leq 2e^{-2\frac{nt^2}{BC^2}}\\
\end{align*}

Let $\lambda_1$ and $\lambda_2$ denote $\iint g(\xx)\kappa^{(\ell_1)}(\xx,\xx')g(\xx') d\xx d\xx'$ and $\iint g(\xx)\kappa^{(\ell_2)}(\xx,\xx')g(\xx') d\xx d\xx'$ respectively. These correspond to the squared norms of the function $g(\cdot)$ under the inner products defined by $\kappa^{(\ell_1)}$ and $\kappa^{(\ell_2)}$ respectively.
 Let us assume that $\lambda_1, \lambda_2 > 0$.  As we prove in the appendix, we can utilize the above inequalities to bound the difference between the ratios $\frac{\frac{1}{n} \sum_{i=1}^n  \phi_1(\xx_i)}{\frac{1}{n} \sum_{i=1}^n \phi_2(\xx_i)}$ and $\frac{\lambda_1}{\lambda_2}$ to  arrive at the following concentration inequality:

\begin{align*}
    &\Prob{\abs{\frac{\frac{1}{n} \sum_{i=1}^n \phi_1(\xx_i)}{\frac{1}{n} \sum_{i=1}^n \phi_2(\xx_i)} - \frac{\lambda_1}{\lambda_2}} \geq \epsilon}\\ 
    &\leq   2e^{-2\frac{n\lambda_1^2\epsilon^2}{AC^2}} + 2e^{-2\frac{n\lambda_2^2\epsilon^2}{BD^2}} 
\end{align*}
Thus we obtain the following theorem:
\begin{thm}
Let $g(x)$ be an arbitrary function with decrements $\lambda_1$ and $\lambda_2$ along kernels $\kappa^{(\ell_1)}$ and $\kappa^{(\ell_2)}$. Let the corresponding decrements in total risk due to a single point $\xx_i$ be given by $\phi_1(\xx_i)$ and $\phi_2(\xx_i)$. Then for $n$ training points, with probability $1-\delta$, we have: 
\begin{align*}
    &\abs{\frac{\frac{1}{n} \sum_{i=1}^n \phi_1(\xx_i)}{\frac{1}{n} \sum_{i=1}^n \phi_2(\xx_i)} - \frac{\lambda_1}{\lambda_2}}\\ &\leq  (\log{\frac{4}{\delta}})\frac{1}{\sqrt{n}}\max\cbr*{\sqrt{\frac{AC^2}{2\lambda_1^2}},\sqrt{\frac{BD^2}{2\lambda_2^2}}}  
\end{align*}
\end{thm}

\subsection{Two layer Network}

We consider a two-layer network with the output for the $i_{th}$ input vector $\xx_i$ given by $f(\xx_i) = \frac{1}{\sqrt{m}}\sum_{j=1}^m v_j\sigma(\bw_j^\top \xx_i)$ with $\bw_j$ denoting the $j_{th}$ row of the $m\times d$ matrix for the first layer.  For simplicity, we consider networks without the bias parameters and assume that all the parameters are initialized independently  as  $w \sim \mathcal N(0,1)$.
The gradients of the output $f_i$ w.r.t the parameters for the two layers are given by:
\begin{equation}
    \df{f(\xx_i)}{v_j} = \frac{1}{\sqrt{m}}\sigma(\bw_j^\top \xx_i)
\end{equation}
and 
\begin{equation}
    \nabla_{\bw_j} f(\xx_i) = \frac{1}{\sqrt{m}}v_j\sigma'(\bw_j^\top \xx_i)\xx_i
\end{equation}

Assuming $v_j \sim \mathcal N(0, 1)$ and $\bw_j \sim \mathcal N(0, \bI) \ \forall j$, we can express the contributions of the two layers to the NTK as follows:
\begin{equation}\label{eq:k0}
\kappa^{(0)}(\xx,\xx') = \Eb{\bw \sim \mathcal N(0, 1)}{\langle \xx, \xx' \rangle \sigma'(\langle \bw, \xx \rangle)\sigma'(\langle \bw, \xx' \rangle)}
\end{equation}

\begin{equation}\label{eq:k1}
    \kappa^{(1)}(\xx, \xx') = \Eb{\bw \sim \mathcal N(0, 1)}{\sigma(\langle \bw, \xx \rangle)\sigma(\langle \bw, \xx' \rangle)}
\end{equation}
\subsubsection{Hermite Polynomials}

Since $\langle \bw, \xx \rangle$ and $\langle \bw, \xx' \rangle$ correspond to correlated gaussian random variables, it is natural to utilize the Hermite expansion for the activations $\sigma(\langle \bw_i, \xx_i \rangle)$. We recall that Hermite polynomials form an orthonormal basis for the $L^2(\mathbb{R},\gamma)$ Hilbert space, where $\gamma$ denotes the one-dimensional Gaussian measure $\gamma(dx) = \frac{1}{\sqrt{2\pi}}e^{-\frac{x^2}{2}}$. For $\sigma(\langle \bw_i, \xx_i \rangle)$, we have a.e. w.r.t the Gaussian measure:
\begin{equation}\label{eq:hermite}
    \sigma(\langle \bw_i, \xx \rangle) = \sum_{i=0}^\infty a_iH_i (\langle \bw_i, \xx \rangle)
\end{equation}
We recall that for random variables $x,y \sim \mathcal N(0,1)$, having correlation $\rho$, we have \citep{daniely}:
\begin{equation}\label{eq:corr}
\begin{split}
    &\Eb{\bw \sim \mathcal N(0, 1)}{H_{n}(x)H_{m}(y)}\\ &=  \begin{cases}
    \rho^n & n=m\\
    0 & \text{otherwise}
    \end{cases}
\end{split}
\end{equation}

Equation \ref{eq:corr} and the bounded convergence theorem imply the following lemma (full proof in the Appendix):
\begin{lemma}\label{lem:inner_prod}
Let $\xi$ be any function admitting a Hermite series expansion $\xi(u) = \sum_{i=0}^\infty e_iH_i (u)$. Then for two random variables $x,y \sim \mathcal N(0,1)$, the correlation $\Ea{\xi(x)\xi(y)}$ only depends on the correlation $\rho$ between $x,y$ and can be expressed as:
\begin{equation}
    \Ea{\xi(x)\xi(y)} = \sum_{i=0}^\infty e_i^2\rho^i
\end{equation}
\end{lemma}

To simplify subsequent computations, we define:
\begin{align*}
    \psi_1(\xx,\xx') = \Eb{\bw \sim \mathcal N(0, 1)}{\sigma'(\langle \bw, \xx \rangle)\sigma'(\langle \bw, \xx' \rangle)}\\
    \psi_0(\xx,\xx') = \Eb{\bw \sim \mathcal N(0, 1)}{\sigma(\langle \bw, \xx \rangle)\sigma(\langle \bw, \xx' \rangle)}
\end{align*}

Since, since $\langle \bw, \xx \rangle, \langle \bw, \xx' \rangle$ are correlated gaussian random variables with correlation $\langle \xx, \xx' \rangle$ and $\psi_0$ and $\psi_1$ correspond to correlations of functions of $\langle \bw, \xx \rangle$ and  $\langle \bw, \xx' \rangle$, Lemma \ref{lem:inner_prod} and the series expansion \ref{eq:hermite} implies that they can be expressed as follows:
\begin{equation}\label{eq:k1_expand}
    \psi_1(\langle\xx, \xx'\rangle) = \sum_{i=1}^\infty a_i^2 (\langle \xx, \xx' \rangle)^i
    \psi_0 = \sum_{i=1}^\infty a_i'^2 (\langle \xx, \xx' \rangle)^i
\end{equation}

Equation \ref{eq:corr} then implies
To relate $\psi_0$ and $\psi_1$, we recall the following recurrence relations:
\begin{align}
    H_{n+1}(w)= w H_{n}(w)- n H_{n-1}(w) \label{eq:rec_1}\\
    H_{n}'(w)= n H_{n-1}(w)\label{eq:rec_2}.
\end{align}

We denote by $a'_i$, the $i_{th}$ coefficient for the Hermite series expansion corresponding to $\sigma'$. We note that, assuming $\lim_{t\rightarrow  \infty}\sigma(t)e^-{\frac{t^2}{2}}=0$, using integration by parts (full proof in the Appendix), we have:
\begin{equation}
\begin{split}
    &a'_n = \frac{1}{n!}\Eb{w \sim \mathcal N(0, 1)}{\sigma'(w) H_{n}(w)}\\
    &= \frac{1}{n!}\Eb{w \sim \mathcal N(0, 1)}{\sigma(w) (w) H_{n}(w)}\\
    &- \frac{1}{n!}\Eb{\bw \sim \mathcal N(0, 1)}{\sigma(w) H_{n}'(w)}\\ &\quad \text{(using Equations \ref{eq:rec_1} and \ref{eq:rec_2})}\\
    &= \frac{1}{n!}\Eb{\bw \sim \mathcal N(0, 1)}{\sigma(w) \rbr*{H_{n+1}(w)}}\\
     &= \frac{n+1}{(n+1)!}\Eb{\bw \sim \mathcal N(0, 1)}{\sigma(w) \rbr*{H_{n+1}(w)}}\\
    &= (n+1)a_{n+1}\\ 
\end{split}
\end{equation}
Therefore, for $\psi_0$, by substituting the above expression in Equation \ref{eq:rec_2}, we obtain:
\begin{equation}\label{eq:k0_expand}
\begin{split}
    \kappa^{(0)}(\langle \xx, \xx' \rangle)  &= \langle \xx, \xx' \rangle \Eb{\bw \sim \mathcal N(0, 1)}{\sigma'(\langle \bw, \xx \rangle)\sigma'(\langle \bw, \xx' \rangle)}\\
    &= \langle \xx, \xx' \rangle \sum_{i=1}^\infty (a'_{i-1})^2(\langle \xx, \xx' \rangle)^{i-1}\\
    &=  \sum_{i=1}^\infty i^2a_i^2 (\langle \xx, \xx' \rangle)^{i}
\end{split}
\end{equation}
Therefore, we have the following theorem:

\begin{thm}\label{thm:two_layer}
Assuming an activation function $\sigma$ such that and $\sigma, \sigma'$ admit Hermite series expansions, the ratio $r(i)$ of the coefficients for the $i_{th}$ degree term in the power series expansion for the Kernels corresponding to layers $1$ and $2$ satisfies $r(i) = i^2$. 
\end{thm}
While our analysis relies on Hermite polynomials, the amplification of high frequency terms due to differentiation is a general phenomemom. For isntance, the derivative of a $k_{th}$ degree sinusoid $sin(kx)$ is given by $k\times cos(kx)$. Thus, we expect a similar analysis to hold in other suitable bases of functions.

\subsection{Conversion to Gegenbauer (Ultraspherical) polynomial series}
Since Gegenbauer polynomials form a basis of polynomial functions on $[-1,1]$, they provide a convenient way to isolate the components of a dot-product kernel corresponding to different degrees of variation w.r.t the input space. Moreover, the Hecke-Funk Theorem allows us to express Gegenbauer polynomials applied to the inner product in terms of the spherical harmonic functions acting on the constituent vectors. This relationship can be utilized to obtain the mercer decomposition of the given dot-product kernel as described in the subsequent sections.
For our results, we utilize the following properties \citep{chen2021deep} of positive definite functions on spheres from the classical paper by \citet{Schoenberg}:
\begin{lemma}\label{lem:sphere_1}
\textbf{Power series expansion}: An inner product Kernel $K$, defined by a continuous function $f$ as $K(\xx,\xx')=f(\xx^\top \xx')$ for $\xx,\xx' \in \bS^{d-1}$ is a positive definite kernel for every $d$ if:
\begin{enumerate}
    \item $f(u)=\sum_{k=0}^\infty a_k u^k$, for a sequence $\{a_k\}_{k=0}^\infty$ satisfying $a_k\ge 0$ and $\sum_{k=0}^\infty a_k < \infty$.
    \item $f(u)=\sum_{k=0}^\infty b_k P_k(u)$, for a sequence $\{b_k\}_{k=0}^\infty$ satisfying $b_k\ge 0$ and $\sum_{k=0}^\infty b_k P_k(1)< \infty$, where $P_k$ denote the Gegenbauer polynomials corresponding to any given dimension.
\end{enumerate}
\end{lemma}

To obtain the Gegenbauer series coefficients from the corresponding power series, we start by expresing $u^i$ in the basis of Gegenbauer polynomials:
\begin{equation}\label{eq:power_legend}
    u^{i}=\sum_{l=0}^\infty \beta^{i}_{l} P_{l}(u)
\end{equation}
By applying part 1 of Lemma \ref{lem:sphere_1} to the function $u^{i}$, we note that $u^{i}$ is a positive definite function $\forall i \geq 0$. Therefore, using part 2 of Lemma \ref{lem:sphere_1}, $\beta^{i}_{l} \geq 0 \quad \forall i,l \geq 0$. Moreover, since $u^{i}$ lies in the span of Gegenbauer polynomials of degree $\leq i$, it is orthogonal to all higher degree Gegenbauer polynomials w.r.t the corresponding measure. Thus $\beta^{i}_{l} = 0 \quad \forall i < l$.
 
Consider a bounded function $g(u)$ admitting a power series expansion with positive coefficients $g(u)=\sum_{i=0}^\infty g_iu^i$ convergent in $(-1,1)$.
Since $P_i(u)$ is bounded and integrable in $[-1,1]$, applying the bounded convergence theorem yields:
\begin{equation}\label{eq:int}
    \int_{-1}^1 g(u) P_l(u)d\mu =  \sum_{m=l}^\infty g_i \int_{-1}^1 u^i P_l(u)d\mu = \sum_{m=l}^\infty g_m  
\end{equation}
where the sum starts from $l$ since $\beta^{i}_{l} = 0 \quad \forall i < l$ and $\mu$ denotes the corresponding dimension-dependent normalizing measure. The details of the measure and the assumptions are further described in the Appendix.
Let us now consider two functions $g(u),h(u)$ with the corresponding power series expansion $g(u)=\sum_{i=0}^\infty g_i u^i$ and $h(u)=\sum_{i=0}^\infty h_i u^i$.
The ratio $\frac{g_{P_l}}{h_{P_l}}$ of the components of the $l_{th}$ degree Gegenbauer polynomials for $g(u)$ and $h(u)$ is then given by:
\begin{equation}\label{eq:ratio_legend}
\begin{split}
\frac{g_{P_l}}{h_{P_l}} &= \frac{\int_{-1}^1 g(u) P_l(u)du}{\int_{-1}^1 h(u) P_l(u)du}\\
&= \frac{\sum_{m=l}^\infty g_m  \beta^{m}_{l}}{\sum_{m=l}^\infty h_m  \beta^{m}_{l}}
\end{split}    
\end{equation}
This leads us to the following lemma:
\begin{lemma}\label{lem:ratio_convert}
Let $g(u),h(u)$ denote two functions on $[-1,1]$ admitting power series expansions $g(u)=\sum_{i=0}^\infty g_i u^i$ and $h(u)=\sum_{i=0}^\infty h_i u^i$ such that $\frac{g_i}{h_i}$ is an non-decreasing function of $i$, then  the ratio of the corresponding Gegenbauer coefficients satisfies $\frac{g_{P_l}}{h_{P_l}} \geq \frac{g_l}{h_l}$.
\end{lemma}
\begin{proof}
\begin{align*}
\frac{g_{P_l}}{h_{P_l}}  &= \frac{\sum_{m=l}^\infty g_m  \beta^{m}_{l}}{\sum_{m=l}^\infty h_m  \beta^{m}_{l}}=\frac{\sum_{m=l}^\infty \frac{g_m}{h_m} h_m \beta^{m}_{l}}{\sum_{m=l}^\infty h_m  \beta^{m}_{l}}\\
&\geq \frac{\sum_{m=l}^\infty  \frac{g_l}{h_l} h_m \beta^{m}_{l}}{\sum_{m=l}^\infty h_m  \beta^{m}_{l}}\\
&= \frac{g_l}{h_l}
\end{align*}
where the inequality follows from the non-decreasing nature of $\frac{g_m}{h_m}$.
\end{proof}
\subsection{Relation to the NTK's Spectrum}

Substituting $g(u)=u\psi_0(u)$ and $h(u)=\psi_1(u)$ in the above Lemma leads to:
\begin{align*}
    \frac{\rbr{u\psi_0(u)}_{P_l}}{\rbr{\psi_1(u)}_{P_l}} \geq \frac{(a'_i)^2}{a_i^2} = i^2.
\end{align*}

We formalize the above conclusion through the following corollary of Theorem  \ref{thm:two_layer} and Lemma \ref{lem:ratio_convert}:
\begin{cor}\label{cor:leg}
For two layer networks with an activation function $\sigma$ such that and $\sigma, \sigma'$ admit Hermite series expansions, let $\lambda^{(1)}_i,\lambda^{(2)}_i$ denote the component along the $i_{th}$ degree Gegenbauer polynomial of the contribution to the NTK by the first and second layer respectively. Then, $\lambda^{(1)}_i \geq i^2 \lambda^{(2)}_i \quad \forall i\geq 0$.
\end{cor}
\subsubsection{Spherical Harmonics}
Since the inner product between two points $\rbr{\langle x', x \rangle}$ lies in the range $[-1,1]$, any dot-product kernel can be expressed in the basis of Gegenbauer polynomials.
 Subsequently, using the Funk-Hecke formula \citep{frye2012spherical}, we can obtain the components of the kernel in the basis of products of the corresponding spherical harmonics, leading to the Mercer decomposition of the kernels. Analogous to the exponential functions on the real line, the spherical harmonics are eigenfunctions of the Laplace–Beltrami operator defined on the unit sphere.
 Let $\lambda_{0,k}, \lambda_{1,k}$ denote the coefficients in the Gegenbauer series expansion for $u\psi_0(u)$ and $\psi_1(u)$ respectively. We utilize the following standard identity, whose proof for arbitrary dimensions can be found in \citet{frye2012spherical}:
 \begin{lemma}\label{lem:funck}
 
 The value of the $k_{th}$ degree Gegenbauer polynomial $P_k$ as a function of the inner product $\langle \xx', \xx \rangle$ of two points $\xx',\xx$ lying on the unit sphere $\mathbb{S}^{d-1}$ can be diagonalized in the basis of the tensor product of spherical harmonics as follows:
 \begin{align*}
P_k(\langle \xx', \xx \rangle)=\frac{1}{N(d,k)}\sum_{j=1}^{N(d,k)} Y_{k,j}(\xx') Y_{k,j}(\xx)
 \end{align*}

 \end{lemma}
 Substituting in the Gegenbauer series expansions for $\psi_0$ and $\psi_1$, we obtain:
\begin{equation}\label{eq:sphharmonic}
\begin{split}
&\langle \xx', \xx \rangle \psi_0(\langle \xx', \xx \rangle)) = \sum_{k=0}^\infty \lambda_{0,k} P_k(\langle \xx', \xx \rangle)\\ 
&= \sum_{k=0}^\infty \lambda_{0,k} \frac{1}{N(d,k)} \sum_{j=1}^{N(d,k)} Y_{k,j}(\xx') Y_{k,j}(\xx) \\
&\psi_1(\langle \xx', \xx \rangle)) = \sum_{k=0}^\infty \lambda_{1,k} P_k(\langle \xx', \xx \rangle)\\ 
&= \sum_{k=0}^\infty \lambda_{1,k} \frac{1}{N(d,k)} \sum_{j=1}^{N(d,k)} Y_{k,j}(\xx') Y_{k,j}(\xx)
\end{split}
\end{equation}
Thus the ratio of the eigenvalues for the two kernels corresponding to a $k_{th}$ degree spherical harmonic is simply given by $\frac{\lambda_{0,k} \frac{1}{N(d,k)}}{\lambda_{1,k} \frac{1}{N(d,k)}} = \frac{\lambda_{0,k}}{\lambda_{1,k}}$
Using the above relationship and corollary \ref{cor:leg}, we have the following corollary:
 \begin{cor} 
 For a functional derivative $d|_{f(t)}$ lying in the eigenspace of $i_{th}$ degree spherical harmonics, the ratio of decrements satisfies $d|_{f(t)} \geq i^2$.
  \end{cor}

\subsection{Convergence Rate while Training Individual Layers}

The spectral bias of the Kernels corresponding to a given layer is related to the rate of convergence along different eigenfunctions while training only the corresponding layer. Concretely, we observe that while training only the parameters of the $\ell_{th}$ layer i.e. $\theta_\ell$, the changes in the output are described by the Kernel $\kappa^{(\ell)}(\xx,\xx') =  \langle\nabla_{\theta_\ell} f(\xx), \nabla_{\theta_\ell} f(\xx') \rangle$
Thus, analogous to Equation \ref{eq:funct}, the evolution of the cost upon training the $\ell_{th}$ layer, while keeping the other layers fixed is given by:
\begin{equation}
\begin{split}
    	\partial_t C|_{f(t)} =-\left\|d|_{f(t)} \right\|_{\kappa^{(\ell)}}^{2},
\end{split}
\end{equation}
where $f_t$ denotes the output function corresponding to the full network. Now, suppose that the cost functional can be decomposed along the eigenfunctions of $\kappa^{(\ell)}$ as  $d|_{f(t)} = \sum_{i=1}^k d|^{(\lambda_i)}_{f(t)}$, where $ d|^{(\lambda_i)}_{f(t)}$ corresponds to the component along the eigenfunction with eigenvalue $\lambda_i$.
Then, due to the orthonormality of the eigenfunctions, we obtain:
\begin{equation}
\begin{split}
    	\partial_t C|_{f(t)} =-\sum_{i=1}^k \lambda_i\left\|d^{(\lambda_i)}|_{f(t)} \right\|_{\kappa^{(\ell)}}^{2},
\end{split}
\end{equation}
Thus, the decrements in the cost can be decomposed along the contributions from different eigenvalues, with the rate of decrease of the corresponding contribution being proportional to the magnitude of the eigenvalue. Our analysis thus predicts that, while training the initial layers, the relative rate of convergence of the function along different eigenfunctions should be larger for higher frequency directions when compared to when training the latter layers.

\section{Experiments}
We empirically verify the validity of our theoretical analysis in both synthetic datasets of spherical harmonics as well as the high dimensional image dataset of MNIST \citep{deng2012mnist}. In all our settings, we measure the norms of different directions in function space, for a given Kernel, as defined in Equation \label{eq:funct} through the quadratic forms on the corresponding gram matrix defined on training points. To compare relative contributions, with further divide each layer's quadratic form by the corresponding value for the last layer. Thus in all our plots, the ``contributions'' denote the ratio of the projections of a given target vector along the gram matrices corresponding to the given layer and the last layer. We include a full definition of the plotted quantities in the Appendix. Additional results for other datasets,  details of the experiments, and experiments for rates of convergence, as discusses in section 5.7, are also provided in the Appendix.
The relative values of the contributions in different settings indicate the prevalence of the layer-wise spectral bias in finite dimensional networks, supporting our analysis. In all our plots, confidence intervals are evaluated over random initializations.

\subsection{Spherical Harmonics}

We plot the layer-wise contributions for spherical harmonics corresponding to input dimension 2 and 10 in Figures \ref{fig:dim2} and \ref{fig:dim10} respectively.
\subsubsection{Two Dimensions}

In two dimensions, the $k_{th}$ degree Gegenbauer polynomial is simply $\cos{k(u)}$, and the spherical harmonics simply correspond to the cosine and sine functions of the $k_{th}$ degree in terms of the polar angle of the input points. 
For this simple setup, we validate our theoretical analysis through experiments on $l2$ regression task with uniformly distributed data on the unit sphere corresponding to dimension $2$. Each input datapoint is thus described as $\xx_\theta =(\cos{\theta},\sin{\theta})$ with $\theta$ being uniformly distributed in the interval $[0,2\pi]$. 
We use a fully connected network with four layers, and consider the quadratic forms for the Kernel gram matrix, evaluated at functions $\cos{\theta},\cos{2\theta},\cos{3\theta},\cos{4\theta}$. To obtain the relative magnitudes of the contributions of different layers for each degree function, we normalize the each contribution by the corresponding contribution of the last layer.

\begin{figure}[h]
    \centering
    \includegraphics[width=\columnwidth]{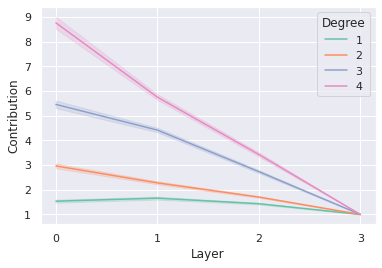}
    \caption{Layer wise contributions for spherical harmonics of different degrees under Relu activation and input dimension 2}
    \label{fig:dim2}
\end{figure}
\vspace{-4mm}
\subsection{Higher Dimensions}
For higher dimensions, we utilize the fact that the  functions of the form $\sum_{l=1}^L a_l P_k(\lin{\xx}{\xx_l})$, for fixed vectors $\xx_l$ are linear combinatations of $k_{th}$ degree spherical harmonics. This can be proved by substituting the expansion of $P_k$ in terms of spherical harmonics, as given by Lemma \ref{lem:funck}:
\begin{align*}
    &\sum_{l=1}^{L} a_l P_k(\langle \xx', \xx \rangle)=\sum_{l=1}^L \frac{a_l}{N(d,k)}\sum_{j=1}^{N(d,k)} Y_{k,j}(\xx_l) Y_{k,j}(\xx)\\
    &= \sum_{j=1}^{N(d,k)} b_j Y_{k,j}(\xx),
\end{align*}
where $b_j = \frac{1}{N(d,k)} \rbr*{\sum_{l=1}^L a_l  Y_{k,j}(\xx_l)}$
Thus, we can sample functions lying in the space spanned by the spherical harmonics of degree $k$ by taking linear combinations of the $k_{th}$ degree Gegenbauer polynomial evaluated at the inner product with randomly sampled points.
.For our experiments, we set all $a_l=1$ consider input dimension $10$ and sample $L=4$ points from the uniform distribution on the unit sphere. Note that, since we evaluate the ratio of contributions, the values remain the same upon scaling the functions with arbitrary constants.

\begin{figure}[h]
    \centering
    \includegraphics[width=\columnwidth]{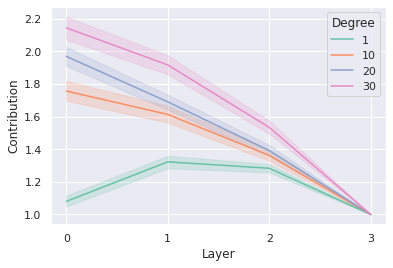}
    \caption{Layer wise contributions for spherical harmonics of different degrees under Relu activation and input dimension 10.}
    \label{fig:dim10}
\end{figure}

In the appendix, we provide additional results for other activation functions and settings.
\subsection{MNIST}

\begin{figure}[h]
    \centering
    \includegraphics[width=\columnwidth]{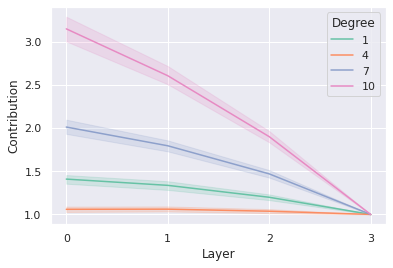}
    \caption{Layer wise contributions of different layers for radial noise of different degrees under Relu activation and the MNIST dataset. }
    \label{fig:mnist_rad}
\end{figure}

While our theoretical analysis is based on inputs distributed on the sphere, we hypothesize that a similar phenomenon of amplification of initial layers' contributions for high frequency directions occurs in the case of more complex high dimensional datasets. 
To validate this, we consider the MNIST dataset \citep{deng2012mnist}, corresponding to images of dimension $28 \times 28$. Since an analysis of the Mercer decomposition for such a dataset is intractable, we utilize the following set of functions to represent different frequencies of variation in the input space.:
\begin{enumerate}
    \item Following \citet{pmlr-v97-rahaman19a}, we consider radial noise of different frequencies, i.e. noise of the form $\sin(k\norm{x})$of different frequencies.
    \item Following \citet{Xu_2020}, we consider the sine and cosine functions of different degrees, defined along the top principal component of the input data.
    \item Unlike the case of uniform distribution on the sphere, the kernels corresponding to contributions from different layers may not be diagonalizable on a common basis of orthonormal eigenfunctions. However, we can approximate such a shared basis for the layer wise Kernels using the eigenvalues for the Kernel corresponding to the full network.
    \end{enumerate}
In each case, we consider a Relu network containing 4 fully connected layers and evaluate the quadratic forms for NTK gram matrix of different layers at functions of different frequency. We provide the results for the first setting in Figure 3, while the results for the remaining settings are provided in the Appendix.

\section{Future Work and Limitations}
While the NTK framework provides insights into the inductive biases of Deep Neural Networks, trained under gradient descent, it's applicability is limited to settings belonging in the ``lazy training'' regime \citep{lazy}. In particular, since the NTK setting assumes that the weights remain near the initialization, it does not directly explain  the properties of feature learning in deep neural networks. However, the kernel and related objects at initialization can still allow us to characterize the initial phase of training.
A promising direction could be to extend our analysis to the spectral properties of intermediate layer near initialization. Extending our analysis to multiple layers is complicated by the presence of products and compositions of power series. Future work could involve avoiding such impediments by utilizing more general characterizations of frequencies in the kernel feature space.

\bibliography{aistats}
\bibliographystyle{abbrvnat}
\pagebreak
\onecolumn
\aistatstitle{
Supplementary Material}
\section{Full Proof of Theorem 1}
We follow the notation in Section 5.3. We recall that Hoeffding's inequality leads to the following two inequalities:
\begin{align*}
    &\Prob{\abs{\frac{1}{n} \sum_{i=1}^n \phi_1(\xx_i)-\lambda_1}
    \geq t}\\&\leq 2e^{-2\frac{nt^2}{AC^2}}\\
    &\Prob{\abs{\frac{1}{n} \sum_{i=1}^n \phi_2(\xx_i)-\lambda_2}
    \geq t}\\ &\leq 2e^{-2\frac{nt^2}{BC^2}}\\
\end{align*}
Our aim is to utilize the above two inequalities to bound $\frac{\frac{1}{n} \sum_{i=1}^n  \phi_1(\xx_i)}{\frac{1}{n} \sum_{i=1}^n \phi_2(\xx_i)}$ and $\frac{\lambda_1}{\lambda_2}$. We proceed by bounding the absolute difference between the ratios in terms of the absolute differences between the numerators and denominators as follows:
\begin{align*}
    &\abs{\frac{\frac{1}{n} \sum_{i=1}^n \phi_1(\xx_i)}{\frac{1}{n} \sum_{i=1}^n \phi_2(\xx_i)} - \frac{\lambda_1}{\lambda_2}}\\
    &= \abs{\frac{\frac{1}{n} \sum_{i=1}^n \phi_1(\xx_i)\lambda_2-\frac{1}{n} \sum_{i=1}^n  \phi_2(\xx_i)\lambda_1}{\frac{1}{n} \sum_{i=1}^n  \phi_2(\xx_i)\lambda_2}}\\
    &= \abs{\frac{\frac{1}{n} \sum_{i=1}^n \phi_1(\xx_i)\lambda_2-\lambda_1\lambda_2+\lambda_1\lambda_2-\frac{1}{n} \sum_{i=1}^n  \phi_2(\xx_i)\lambda_1}{\frac{1}{n} \sum_{i=1}^n  \phi_2(\xx_i)\lambda_2}}\\
    &\leq \frac{\abs{\frac{1}{n} \sum_{i=1}^n \phi_1(\xx_i)-\lambda_1}}{\frac{1}{n} \sum_{i=1}^n  \phi_2(\xx_i)}+\frac{\abs{\frac{1}{n} \sum_{i=1}^n  \phi_2(\xx_i)-\lambda_2}\lambda_1}{\frac{1}{n} \sum_{i=1}^n  \phi_2(\xx_i)\lambda_2}
\end{align*}
Next, we observe that to ensure that the above difference is bounded by $\epsilon\frac{\lambda_1}{\lambda_2}$ it is sufficient to have:
\begin{align*}
    \frac{\abs{\frac{1}{n} \sum_{i=1}^n \phi_1(\xx_i)-\lambda_1}}{\frac{1}{n} \sum_{i=1}^n  \phi_2(\xx_i)} \leq \frac{\epsilon}{2}\frac{\lambda_1}{\lambda_2}\\
    \frac{\abs{\frac{1}{n} \sum_{i=1}^n  \phi_2(\xx_i)-\lambda_2}\lambda_1}{\frac{1}{n} \sum_{i=1}^n  \phi_2(\xx_i)\lambda_2} \leq \frac{\epsilon}{2}\frac{\lambda_1}{\lambda_2}
\end{align*}
Now suppose we have $\abs{\frac{1}{n} \sum_{i=1}^n  \phi_2(\xx_i)-\lambda_2} \leq \frac{\epsilon\lambda_2}{4}$ and$\abs{\frac{1}{n} \sum_{i=1}^n \phi_1(\xx_i)-\lambda_1} \leq \frac{\epsilon\lambda_1}{4}$ for some $\epsilon \leq \frac{1}{2}$ Then the ratios can be bounded as follows:
\begin{align*}
    &\frac{\abs{\frac{1}{n} \sum_{i=1}^n \phi_1(\xx_i)-\lambda_1}}{\frac{1}{n} \sum_{i=1}^n  \phi_2(\xx_i)}\\
    &\leq \frac{\epsilon\lambda_1}{4\lambda_2(1-\frac{\epsilon}{4})}\\
    &\leq \frac{\lambda_1}{\lambda_2}\frac{\epsilon}{4}(1+\frac{\epsilon}{2})
     \leq \frac{\lambda_1}{\lambda_2}\frac{\epsilon}{2}\\
     &\frac{\abs{\frac{1}{n} \sum_{i=1}^n \phi_2(\xx_i)-\lambda_2}}{\frac{1}{n} \sum_{i=1}^n  \phi_2(\xx_i)}\frac{\lambda_1}{\lambda_2}\\ &\leq \frac{\epsilon\lambda_2}{4\lambda_2(1-\frac{\epsilon}{4})\lambda_2}\frac{\lambda_1}{\lambda_2}\\ 
    &\leq \frac{\lambda_1}{\lambda_2}\frac{\epsilon}{4}(1+\frac{\epsilon}{2})
     \leq \frac{\lambda_1}{\lambda_2}\frac{\epsilon}{2}
\end{align*}
Thus, using union bound, we obtain:
\begin{align*}
    &\Prob{\abs{\frac{\frac{1}{n} \sum_{i=1}^n \phi_1(\xx_i)}{\frac{1}{n} \sum_{i=1}^n \phi_2(\xx_i)} - \frac{\lambda_1}{\lambda_2}} \geq \epsilon}\\ 
    &\leq   2e^{-2\frac{n\lambda_1^2\epsilon^2}{AC^2}} + 2e^{-2\frac{n\lambda_2^2\epsilon^2}{BD^2}} 
\end{align*}

Finally, we observe that for $\epsilon \leq (\log{\frac{4}{\delta}})\frac{1}{\sqrt{n}}\max\cbr*{\sqrt{\frac{AC^2}{2\lambda_1^2}},\sqrt{\frac{BD^2}{2\lambda_2^2}}}$, we have $2e^{-2\frac{n\lambda_1^2\epsilon^2}{AC^2}} + 2e^{-2\frac{n\lambda_2^2\epsilon^2}{BD^2}} \leq \delta$. This directly leads to the statement of Theorem 1.

\section{Proof of Lemma 1}

We have $\xi(x) = \sum_{i=0}^\infty e_iH_i (x)$ and $\xi(y) = \sum_{i=0}^\infty e_iH_i (y)$. Thus
\begin{align*}
    \Ea{\xi(x)\xi(y)} &= \Ea{(\sum_{i=0}^\infty e_iH_i (x))(\sum_{j=0}^\infty e_iH_i (y))}\\ 
    &=\sum_{i=0}^\infty\sum_{j=0}^\infty\Ea{e_ie_jH_i (x)H_j (y)} = \sum_{i=0}^\infty e_i^2\rho^i.
\end{align*}

We recall that the Hermite polynomials $H_i$ form a basis for an infinite dimensional inner product (Hilbert) space. Thus we have $\sum_{i=0}^\infty e_i^2 \leq \infty$. Moreover, using Cauchy-Shwartz, we further have $\Ea{H_i (x)H_j (y)} \leq  \Ea{H_i (x)^2}^{\frac{1}{2}}\Ea{H_i (y)^2}^\frac{1}{2} = 1$. Thus the bounded convergence theorem allows expressing the above integral as an infinite series. Finally, using Equation 11 yields  $\Ea{\xi(x)\xi(y)} = \sum_{i=0}^\infty e_i^2\rho^i$.

\section{Full Derivation for the Hermite Coefficients of $\sigma'$}

We utilize the recurrence relations for $H_n$ and integration by parts as follows:
\begin{equation}
\begin{split}
    &a'_n = \frac{1}{n!}\Eb{w \sim \mathcal N(0, 1)}{\sigma'(w) H_{n}(w)}\\
    &= \frac{1}{n!}\rbr*{\int_{-\infty}^\infty \frac{1}{\sqrt{2\pi}} \sigma'(w) H_{n}(w) e^{-\frac{w^2}{2}} dw}\\
    &= \frac{1}{n!}\rbr*{\frac{1}{\sqrt{2\pi}}\sigma(w) H_{n}(w) e^{-\frac{w^2}{2}}\vert_{-\infty}^\infty }\\
    &- \frac{1}{n!}\rbr*{\int_{-\infty}^\infty \frac{1}{\sqrt{2\pi}} \sigma(w) \frac{d\rbr{H_{n}(w)e^{-\frac{w^2}{2}}}}{dw} dw}\\
    &= - \frac{1}{n!}\rbr*{\int_{-\infty}^\infty \frac{1}{\sqrt{2\pi}} \sigma(w) \frac{d\rbr{H_{n}(w)e^{-\frac{w^2}{2}}}}{dw} dw}\\
    &= \frac{1}{n!}\Eb{w \sim \mathcal N(0, 1)}{\sigma(w) (w) H_{n}(w)}\\
    &- \frac{1}{n!}\Eb{\bw \sim \mathcal N(0, 1)}{\sigma(w) H_{n}'(w)} \quad \text{(using Equations 14 and 15}\\
    &= \frac{1}{n!}\Eb{\bw \sim \mathcal N(0, 1)}{\sigma(w) \rbr*{H_{n+1}(w)}}\\
     &= \frac{n+1}{(n+1)!}\Eb{\bw \sim \mathcal N(0, 1)}{\sigma(w) \rbr*{H_{n+1}(w)}}\\
    &= (n+1)a_{n+1}\\ 
\end{split}
\end{equation}
\section{Gegenbauer Coefficients in $d$ dimensions and Proof of Equation \ref{eq:int}}

The Gegenbauer polynomials in dimension $d$ form an orthonormal basis for the Hilbert space $L_2([-1,1],\nu)$ where $\nu(u) = (1-u^2)^\frac{(d-3)}{2}$. Note that the corresponding polynomials $P_k$ are continuous, and hence bounded in the compact interval $[-1,1]$ by a constant say $C_k$. Moreover, by assumption, all the coefficients $g_i$ in the power series expansion $g(u)=\sum_{i=0}^\infty g_iu^i$ are positive. Therefore, for $u \in [-1,1]$, $g(u) =\sum_{i=0}^\infty g_iu^i \leq \sum_{i=0}^\infty g_i = 1$. Subsequently  $\sum_{i=0}^j g_iu^i P_k(u) \leq C_k g(1)$.  Thus, using the bounded convergence theorem, we have:
\begin{align*}
    \int_{-1}^1 g(u) P_l(u)d\mu &= \sum_{m=0}^\infty g_i \int_{-1}^1 u^i P_l(u)d\mu\\
    &=\sum_{m=l}^\infty g_i \int_{-1}^1 u^i P_l(u)d\mu = \sum_{m=l}^\infty g_m  \beta^{m}_{l},
\end{align*}
where the second inequality follows from the observation that all polynomials of degree $< l$  lie in the span of $P_m$ with $m < l$ and are thus orthogonal to $P_l(u)$ to w.r.t the measure $\nu$.

\section{Experiments for Convergence under Layer-wise Training}
Consider a target function of the form $y^* = \sum_{i=1}^k y^*_i$, where $y^*_i$ is a function of degree $i$, with $y^*_i, y^*_j$ being orthogonal w.r.t the kernel $\kappa_{NTK}$ for $i\neq j$. For a given set of $n$ training points, we let $\yy,\yy^*_i,\yy^*$ denote the $n$ dimensional vectors containing the outputs of the network and the values of the functions $\yy^*_i,\yy^*$ evaluated at these $n$ points.
We define the residual component $r_i$ for degree $i$ as the squared projection of the residual $(\yy-\yy^*)$ along the $i_{th}$ degree component $\yy_i^*$ i.e $r_i=\frac{((\yy-\yy^*)^T(\yy_i^*))^2}{\norm{\yy_i^*)}^2}$. 
For our experiments, we consider $y^*_i$ corresponding to spherical harmonics of different degrees with input dimension $10$. We use a four layer fully connected network with each layer of width $100$.

\begin{figure}
    \subfigure{\includegraphics[width=0.5\textwidth]{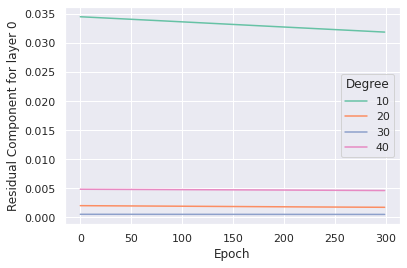}}
    \subfigure{\includegraphics[width=0.5\textwidth]{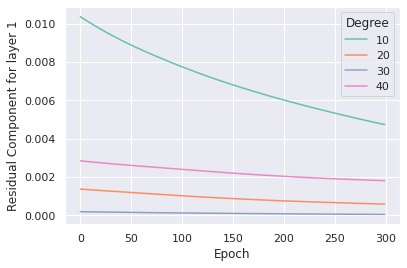}}
    \subfigure{\includegraphics[width=0.5\textwidth]{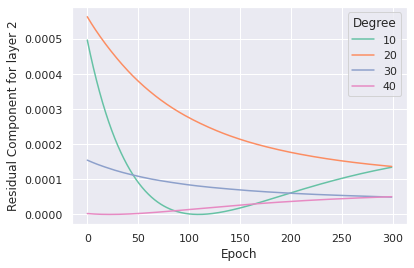}}
    \subfigure{\includegraphics[width=0.5\textwidth]{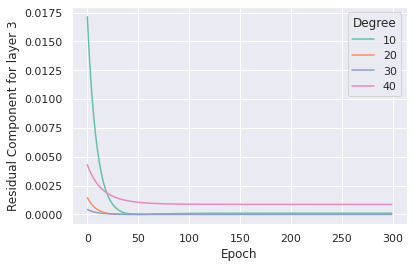}}
\caption{Evolution of the Residual Components corresponding to different degrees when individually training different layers}
\end{figure}

From  the above figure, we observe that the initial layers have the least residual component for the high degree terms, and lead to significantly slower convergence for the low degree terms. Contrarily, the latter layers lead to faster convergence for the low degree terms and have relatively high residual components for the high degree terms.

\section{Definition of Contribution in Empirical Results}

For a given set of training points $\xx_1,\cdots,\xx_n$ and a kernel $\kappa^{(\ell)}$, corresponding to the $\ell_{th}$ layer, let $G^{(\ell)}$ denote the gram matrix with entries $G^{(\ell)}_{i,j} = \kappa^{(\ell)}(\xx_i,\xx_j)$. Then, for a given function $f(\xx)$, having values $\yy=(f(\xx_1),\cdots, f(\xx_n))^\top$ at the $n$ training points, we define the contribution $c_l$ of the $\ell_{th}$ layer to decrements in the training error along $f$ as:
\begin{align*}
    c_l = \yy^\top G^{(\ell)} \yy^\top.
\end{align*}
Finally, we normalize the contribution of the $\ell_{th}$ layer by the contribution of the last layer i.e $c_{L-1}$ to obtain the relative contributions $\frac{\yy^\top G^{(\ell)} \yy^\top}{\yy^\top G^{(L-1)} \yy^\top}$. We plot these relative contributions for different datasets and different functions $f$ in all the figures.

\section{Additional Results and Details for MNIST}

Here we include the results for settings 2 and 3 described in section 6.3  in Figures \ref{fig:pca} and \ref{fig:full} respectively. In all the three settings, we flatten the input into a vector of dimension $784$ and use four layer fully-connected ReLU networks having the same width i.e $784$. For the computation of the PCA and the gram matrices, we use $1000$ randomly sampled training points.

\begin{figure}
    \centering
    \includegraphics{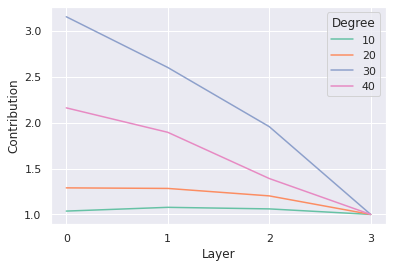}
    \caption{Layer wise contributions of different layers for sinusoids of different degrees along the first principal component under Relu activation and the MNIST dataset.}
    \label{fig:pca}
\end{figure}
\begin{figure}
    \centering
    \includegraphics{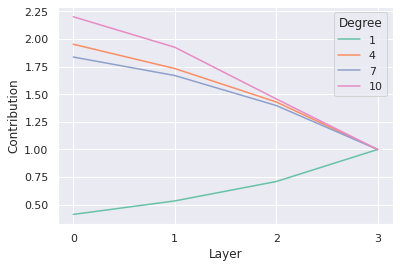}
    \caption{Layer wise contributions of different layers for different eigenvectors of the full Kernel under Relu activation and the MNIST dataset. Here degree denotes the rank of the magnitude of the eigenvalues when ordered from largest to smallest.}
    \label{fig:full}
\end{figure}
\section{Additional Results for CIFAR}

Using the same setup as the MNIST dataset, we plot the relative contributions for the CIFAR dataset under setting 2 i.e. by considering sinusoid functions of different degrees along the projection defined by the first principle component of CIFAR. We plot the results in Figure \ref{fig:cifar}, which again support the validity of our hypothesis in high dimensional datasets.

\begin{figure}[h]
    \centering
    \includegraphics{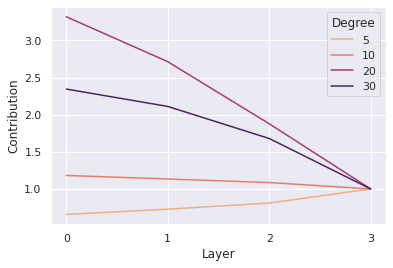}
    \caption{Relative contributions for the CIFAR dataset.}
    \label{fig:cifar}
\end{figure}

\section{Results for Tanh activation Function}

We provide results for the Tanh activation function under the same setting as Section 6.2 i.e spherical harmonics of different degrees for input dimension 10.

\begin{figure}
    \centering
    \includegraphics{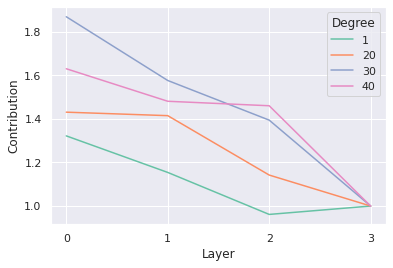}
    \caption{Layer wise contributions for spherical harmonics of different degrees under Tanh activation and input dimension 10.}
    \label{fig:tanh}
\end{figure}
From Figure \ref{fig:tanh}, we observe that the amplification of the contributions of initial layers for high degree functions is still present, while being much less pronounced than for ReLU. We believe this is due to the much faster decay of the gradient and the high frequency terms for Tanh, which prevents the gradient from effectively propagating high frequency signals to the initial layers.

\end{document}